%% file: lrpca.tex
\documentclass[11pt,letterpaper]{article}

\usepackage[top=1in, bottom=1in, left=1in, right=1in]{geometry}  
\usepackage{amsmath, amsthm, amsfonts, bm}  
\usepackage{mathtools}   
\usepackage{microtype}
\usepackage{graphicx}
\usepackage{subfigure}
\usepackage{tikzsymbols,tikz,pgfplots}

\usepackage{booktabs} 

\usepackage{hyperref}

\usepackage{algorithm}
\usepackage{algorithmic}

\usepackage{authblk} 

\usepackage{enumerate}
\usepackage{enumitem}   
\usepackage[authoryear]{natbib}  
\usepackage{url} 
\usepackage{hyperref}
\hypersetup{
    colorlinks=true,
    citecolor=blue,
    linkcolor=blue,
    filecolor=magenta,      
    urlcolor=cyan,
}

\usepackage{arydshln}  
\usepackage{pdflscape}
\usepackage{afterpage}
\usepackage{capt-of}
\usepackage{adjustbox}

\usepackage{multirow}   
\usepackage{tabularx} 

\usepackage{float}   
\usepackage{nameref}   

\usepackage[titletoc,title]{appendix}  
\usepackage{longtable} 

\graphicspath{ {figs/} }
\input{ccMacro.tex}

\newcommand{\mybibsty}{chicago}
\newcommand{\mybibfile}{lrpca}

\newcommand{\blind}{0}

\begin{document}

\if0\blind
{
	\title{\bf  Low-Rank Principal Eigenmatrix  Analysis}
	\author[1]{Krishna Balasubramanian}
	\author[2]{Elynn Y. Chen 
	}
	\author[3]{Jianqing Fan}
	\author[4]{Xiang Wu}
	\affil[1]{Department of Statistics, The University of California, Davis}
	\affil[2,3]{Department of Operations Research and Financial
		Engineering, Princeton University}
	\affil[4]{Google AI.}
	\date{\vspace{-5ex}}
	\maketitle
} \fi

\if1\blind
{
	\bigskip
	\bigskip
	\bigskip
	\title{\bf Low-Rank Principal Eigenmatrix  Analysis }
	\author{}
	\date{\vspace{-5ex}}
	\maketitle
	\medskip
} \fi

\bigskip

\maketitle

\begin{abstract}
	Sparse PCA is a widely used technique for high-dimensional data analysis. In this paper, we propose a new method called low-rank principal eigenmatrix analysis. Different from sparse PCA, the dominant eigenvectors are allowed to be dense but are assumed to have low-rank structure when matricized appropriately. Such a structure arises naturally in several practical cases: Indeed the top eigenvector of a circulant matrix, when matricized appropriately is a rank-1 matrix.  We propose a matricized rank-truncated power method that could be efficiently implemented and establish its computational and statistical properties. Extensive experiments on several synthetic data sets demonstrate the competitive empirical performance of our method. 
\end{abstract}

\input{intro}

\input{TPM}

\input{empirical}

\section{Discussion} \label{sec:summary}
In this paper, we propose a novel Low-rank Principal Eigenmatrix Analysis technique for data analysis which is formulated as a non-convex optimization problem. We propose the~\textsc{smart-pc} algorithm for solving the above formulation and established its computational and statistical properties. 

Our method opens up a lot of potential directions for follow-up work. First, the question of obtaining an initializer satisfying the condition of Theorem 1 or proving the nonexistence of a obtaining such an initializer in polynomial time is extremely interesting. It is also interesting to develop  convex relaxations for the proposed non-convex problem that works provably well. Next, analyzing the general permuted low-rank eigenmatrix model, discussed in~\S\ref{sec:model} is challenging.  Finally, note that in this work we only considered matricization of the top-eigenvector. One could potentially also consider tensorizations of the top-eigenvector and enforce structures on that. Indeed, we also noted empirically that the eigenvectors when tensorized also have a certain low-rank property. We plan to investigate these directions for future work.

\clearpage
\bibliographystyle{\mybibsty}
\bibliography{\mybibfile}

\clearpage
\input{appendix}

\end{document}

%% file: ccMacro.tex
\usepackage{mathtools}

\newcommand{\myol}[2][3]{{}\mkern#1mu\overline{\mkern-#1mu#2}}

\newcommand{\utwi}[1]{\mbox{\boldmath $ #1$}}

\newcommand{\ba}{{\utwi{a}}}

\newcommand{\bb}{{\utwi{b}}}

\newcommand{\bp}{{\utwi{p}}}
\newcommand{\bq}{{\utwi{q}}}

\newcommand{\bu}{{\utwi{u}}}
\newcommand{\bv}{{\utwi{v}}}
\newcommand{\bw}{{\utwi{w}}}
\newcommand{\bx}{{\utwi{x}}}
\newcommand{\by}{{\utwi{y}}}
\newcommand{\bz}{{\utwi{z}}}
\newcommand{\bA}{{\utwi{A}}}

\newcommand{\bC}{{\utwi{C}}}
\newcommand{\bD}{{\utwi{D}}}
\newcommand{\bE}{{\utwi{E}}}

\newcommand{\bP}{{\utwi{P}}}
\newcommand{\bQ}{{\utwi{Q}}}

\newcommand{\bU}{{\utwi{U}}}
\newcommand{\bV}{{\utwi{V}}}
\newcommand{\bW}{{\utwi{W}}}
\newcommand{\bX}{{\utwi{X}}}
\newcommand{\bY}{{\utwi{Y}}}

\newcommand{\bbx}{\bar\bx}

\newcommand{\bbA}{\myol{\bA}}

\newcommand{\bbU}{\bar\bU}
\newcommand{\bbV}{\bar\bV}

\newcommand{\bLambda}{{\utwi{\Lambda}}}

\newcommand{\bSigma}{{\utwi{\Sigma}}}

\newcommand{\bbX}{\myol{\bX}}

\newtheorem{definition}{Definition}[section]
\newtheorem{theorem}{Theorem}
\newtheorem{lemma}{Lemma} 

\newtheorem{proposition}{Proposition}
\newtheorem{assumption}{Assumption}
\newtheorem{example}{Example}
\newtheorem{remarkx}{Remark}
\newenvironment{remark}{\begin{remarkx} \em}{\end{remarkx}}

\newtheorem{condx}{Condition}

\renewcommand{\hat}{\widehat}

\DeclarePairedDelimiterX{\norm}[1]{\lVert}{\rVert}{#1}   

\newcommand{\E}[1]{{\rm E} \left\{ #1 \right\}}    

\newcommand{\rank}[1]{{\rm rank} \left( #1 \right)}
\newcommand{\vect}[1]{{\textsc{vec}} \left( #1 \right)}
\newcommand{\mat}[1]{{\textsc{mat}} \left( #1 \right)}

\newcommand{\diag}[1]{{\rm diag} \left( #1 \right)}

\newcounter{question}

\newcounter{misc}

%% file: intro.tex
\section{Introduction}
\label{sec:intro}
Principal Component Analysis (PCA) is a popular technique for analyzing big data with a rich theoretical literature, along with several real-world applications in analyzing financial data, matrix completion, network analysis, gene expression analysis, to name a few. We refer the reader to~\cite{izenman2008modern,jolliffe2011principal,fan2018principal} for a recent survey. Given a $d \times d$ symmetric positive semidefinite matrix $\bA$, PCA finds the unit vector $\bx \in \mathbb{R}^d$, also called as the Principal Component (PC), that maximizes the quadratic form $\bx^\top \bA \bx$. In statistical machine learning, $\bA$ is often taken to be the sample covariance matrix of a $d$-dimensional random vector, estimated with $n$ sample points. Several theoretical results concerning the consistency and rates of convergence of the estimated principal components are well-known in the setting when $d$ is fixed and $n \to \infty$. See for example~\cite{anderson1963asymptotic, muirhead2009aspects}.

On the other hand, contemporary datasets often have the number of input variables ($d$) comparable with or even much larger than the number of samples ($n$). Asymptotic properties of the PCs under the setting that $n \to \infty$ such that $d/n \to c \in (0,\infty)$ have been studied by several authors; See for example~\cite{nadler2008finite, jung2009pca, johnstone2012consistency, johnstone2018pca}. At the risk of sounding non-rigorous, the main conclusion of that line of work is that the estimated PC is not a consistent estimator of the true PC in this setting. Such a concern lead to the development of \textit{sparse PCA}~\cite{zou2006sparse}, where it is assumed that only $s$ of the $d$ components in the unit vector $\bx$ are non-zero. Such a technique, improves the interpretation of the PC and alleviates some of the problems associated with the consistency of PCA in high-dimensional low-sample size scenario~\cite{birnbaum2013minimax, vu2012minimax,vu2013minimax,cai2013sparse}. But the question, \emph{when are the PCs of a covariance matrix truly sparse?}, is not well understood. So far, the main motivation for sparse PCA has been the fact that high-dimensional estimation is possible only under the structural sparsity assumptions, which is rather unconvincing. To the best of our knowledge,~\cite{lei2015sparsistency} is the only work that analyzes sparse PCA in an agnostic setting, i.e., when the true PC is not assumed to be sparse.

In this paper we propose a novel method called Low-Rank Principal Eigenmatrix Analysis, where the PC is assumed to be \emph{low-rank} when \emph{matricized} appropriately. A main motivation for proposing such a structure is from the case of circulant and Toeplitz covariance matrices that arises in the context of stochastic processes used to model financial applications and medical imaging dataset~\cite{christensen2007algorithm, snyder1989use, roberts2000hidden, cai2013optimal}. Consider the case when $d=p^2$. It is known that circulant matrices have FFT matrices as eigenvectors~\cite{gray2006toeplitz, davis2012circulant}. 


\begin{example}[Circulant Matrix]
Theorem 3.1 in \cite{gray2006toeplitz} shows that every circulant matrix $\bC \in \mathbb{R}^{p^2 \times p^2}$ has eigenvectors of the form $\by = (1, \rho, \cdots, \rho^{p^2-1})^\top \in \mathbb{R}^{p^2}$. If we reshape the vector $\by$ according to Definition~\ref{dfn:mat}, then the matrix $\mat{\by} \in \mathbb{R}^{p \times p}$ is an exactly rank one matrix, i.e. $\mat{\by} = \bu \bv^\top$ where $\bu=\begin{pmatrix}1&\rho&\cdots&\rho^p\end{pmatrix}^\top$ and $\bv=\begin{pmatrix}1&\rho^{p+1}&\cdots&\rho^{p(p-1)+1}\end{pmatrix}^\top$.
\end{example}

\begin{example}[Toeplitz, Diagonally Dominant and Kronecker Structred Covariance Matrices]
In Figure~\ref{fig:justification}, we plot the spectrum of the matricized version of top eigenvector of a Toeplitz, diagonally dominant, kronecker-product structured covariance matrices~\cite{werner2008estimation}, along that of a unstructured covariance matrix. The dimensions of the covariance matrix is  $10000 \times 10000$. The top eigenvector is hence of dimension $10000$. The reshaping according to Definition~\ref{dfn:mat}, leads to a $100 \times 100$ matrix. We see that, for the case of Toeplitz matrices and Diagonally Dominant matrices, the top eigenmatrix could almost always be well-aproximated by an extremely low-rank matrix. Furthermore, for covariance matrices with Kronecker product structure, the rank of the top eigenmatrix is well approximated by a low-rank approximation: for a $10000$ dimensional principal component represented as a $100 \times 100$ matrix a rank-40 approximation captures the bulk. Note that this obviously does not hold in the case of general covariance matrices.
\end{example}
A noteworthy observation in the above examples is that while the reshaped eigenvectors low-rank, they are never sparse. Thus in these cases, while the sparsity assumption on the eigenvector might appear rather superficial, the low-rank assumption appears naturally. This observation serves as a motivation for the low-rank model proposed in this work. 
\begin{figure*}[ht]\label{fig:justification}
\centering
\begin{minipage}{0.5\textwidth}
\centering
\begin{tikzpicture}[scale=0.70]
  \begin{axis}[
  xlabel = $\text{Index of Singular Values}$,
ylabel=$\text{Normalized Magnitude of Singular Values}$,
xmax = 10,
xmin = 1,
ymax = 1,
ymin = 0
]
\addplot+[error bars/.cd,
y dir=both,y explicit]
 coordinates {
( 1, 1 )+- (0.0, 0.07)
( 2, 0.03 )+- (0.0, 0.05)
( 3, 0.02 )+- (0.0, 0.03)
( 4, 0.00 )+- (0.0, 0.00)
( 5, 0.00 )+- (0.0, 0.00)
( 6, 0.00 )+- (0.0, 0.00)
( 7, 0.00 )+- (0.0, 0.00)
( 8, 0.00 )+- (0.0, 0.00)
( 9, 0.00 )+- (0.0, 0.00)
( 10, 0.00 )+- (0.0, 0.00)
};  
\end{axis}
\end{tikzpicture}
\end{minipage}\hfill
\begin{minipage}{0.5\textwidth}
\centering
\begin{tikzpicture}[scale=0.70]
  \begin{axis}[
 xlabel = $\text{Index of Singular Values}$,
ylabel=$\text{Normalized Magnitude of Singular Values}$,
xmax = 10,
xmin = 1,
ymax = 1,
ymin = 0
]
\addplot+[error bars/.cd,
y dir=both,y explicit]
 coordinates {
( 1, 1 )+- (0.0, 0.07)
( 2, 0.06 )+- (0.0, 0.06)
( 3, 0.03 )+- (0.0, 0.05)
( 4, 0.02 )+- (0.0, 0.03)
( 5, 0.00 )+- (0.0, 0.00)
( 6, 0.00 )+- (0.0, 0.00)
( 7, 0.00 )+- (0.0, 0.00)
( 8, 0.00 )+- (0.0, 0.00)
( 9, 0.00 )+- (0.0, 0.00)
( 100, 0.00 )+- (0.0, 0.00)
};  
\end{axis}
\end{tikzpicture}
\end{minipage}\hfill
\begin{minipage}{0.5\textwidth}
\centering
\begin{tikzpicture}[scale=0.70]
  \begin{axis}[
  xlabel = $\text{Index of Singular Values}$,
ylabel=$\text{Normalized Magnitude of Singular Values}$,
xmax = 100,
xmin = 1,
ymax = 1,
ymin = 0
]
\addplot+[error bars/.cd,
y dir=both,y explicit]
 coordinates {
( 1, 1 )+- (0.0, 0.007)
( 20, 0.83 )+- (0.0, 0.06)
( 30, 0.43 )+- (0.0, 0.05)
( 40, 0.22 )+- (0.0, 0.03)
( 50, 0.00 )+- (0.0, 0.04)
( 60, 0.00 )+- (0.0, 0.00)
( 70, 0.00 )+- (0.0, 0.00)
( 80, 0.00 )+- (0.0, 0.00)
( 90, 0.00 )+- (0.0, 0.00)
( 100, 0.00 )+- (0.0, 0.00)
};  
\end{axis}
\end{tikzpicture}
\end{minipage}\hfill
\begin{minipage}{0.5\textwidth}
\centering
\begin{tikzpicture}[scale=0.70]
  \begin{axis}[
  xlabel = $\text{Index of Singular Values}$,
ylabel=$\text{Normalized Magnitude of Singular Values}$,
xmax = 100,
xmin = 1,
ymax = 1,
ymin = 0
]
\addplot+[error bars/.cd,
y dir=both,y explicit]
 coordinates {
( 1, 1 )+- (0.0, 0.007)
( 20, 0.86 )+- (0.0, 0.06)
( 30, 0.73 )+- (0.0, 0.05)
( 40, 0.62 )+- (0.0, 0.03)
( 50, 0.45 )+- (0.0, 0.04)
( 60, 0.40 )+- (0.0, 0.03)
( 70, 0.30 )+- (0.0, 0.03)
( 80, 0.20 )+- (0.0, 0.03)
( 90, 0.10 )+- (0.0, 0.03)
( 100, 0.01 )+- (0.0, 0.00)
};  
\end{axis}
\end{tikzpicture}
\end{minipage}\hfill
\caption{Singular values of  top eigenmatrix of a mean zero multivariate Gaussian random variable generated with the specified covariance structure, averaged over $100$ instance. Top Left:  Toeplitz Covariance Matrix. Top Right: Diagonally Dominant Covariance matrix. Bottom Left: Kronecker Covariance. Bottom Right: General PSD covariance matrix. Notice also the difference in the scale of x-axis. Here, by normalized, we refer to normalizing by the largest singular value.}
\end{figure*}
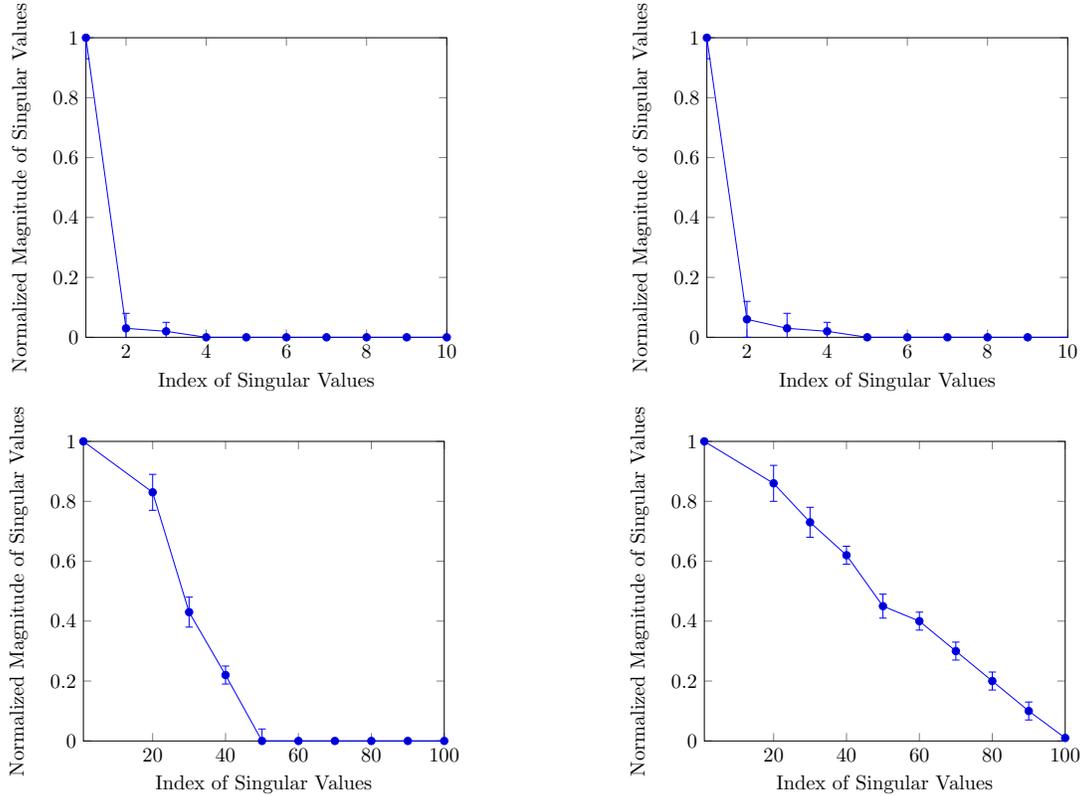

\noindent\textbf{Organization:} The remaining of this paper is organized as follows: \S\ref{sec:modelmain} describes the low-rank principal eigenmatrix problem and proposes a iterative algorithm for computing the estimator. \S\ref{sec:MatRankTrunc} analyzes the convergence rates of the proposed algorithm. \S\ref{sec:empirical} provides numerical results that confirm the relevance of our theoretical predictions. We end this section with the list of notations we follow in the rest of the paper.

\noindent\textbf{Notation:} We use bold-faced letter $\by$ to denote vectors, and bold-faced capital letters such as $\bA$ to denote matrices. For a vector $\by$, $\norm{\by}_{\ell_0}$ and $\norm{\by}_{\ell_2}$ represents the $L_0$ and $L_2$ norms respectively. For any matrix $\bX \in \mathbb{R}^{d_1 \times d_2}$, we denote its singular values by $s_1(\bX) \ge \cdots \ge s_d(\bX) \ge 0$ where $d=\min(d_1, d_2)$. The inner product in $\bX \in \mathbb{R}^{d_1 \times d_2}$ can be written in matrix form as $\langle \bX, \bY \rangle = \mathbf{Tr}(\bX^\top \bY)$. Let $\norm{\bX}_F$, $\norm{\bX}$ and $\norm{\bX}_{\star}$ denote the Frobenius, operator and nuclear norm of $\bX$, respectively. Let $\mathbb{S}^{d \times d}$ denote the set of symmetric matrices. For any matrix $\bA \in \mathbb{S}^{d \times d}$, we denote its eigenvalues by $\lambda_{\max}(\bA) = \lambda_1(\bA) \ge \cdots \ge \lambda_d(\bA) = \lambda_{\min}(\bA)$. Furthermore, we call the eigenvector associated with the top eigenvalue $\lambda_{\max}(A)$ as the top eigenvector. We use $\rho(\bA)$ to denote the spectral norm of $\bA$, i.e. $\max\{|\lambda_{\max}(\bA)|,|\lambda_{\min}(\bA)|\}$. In the rest of the paper, we define $Q(\bx) \coloneqq \bx^\top \bA \bx$, $P_U$ as the projection matrix onto the column space of matrix $\bU$, and $\bA_{UV} = P_{V\otimes U} \bA P_{V\otimes U}$. We next define \textit{vectorization} of a matrix and \textit{matricization} of a vector in Definition \ref{dfn:vect} and \ref{dfn:mat} respectively.

\begin{definition}[Vectorization] \label{dfn:vect}
For a matrix $\bX \in \mathbb{R}^{p \times p}$, we define the vectorization of the matrix  $\vect{\bX} = \mathbb{R}^{p^2}$ to be a vector constructed by stacking columns of $\bX$ together.
\end{definition}

\begin{definition}[Matricization] \label{dfn:mat}
Given a vector $\bx \in \mathbb{R}^{p^2}$, we define the matricization of a vector by $\mat{\bx}: \mathbb{R}^{p^2} \mapsto \mathbb{R}^{p \times p}$ where the first $p$ coordinates of the vectors form the first column of the matrix and second $p$ indices form the second column and so on. Hence $\bX_{i,j} = \bx_{(j-1)p +i}$.
\end{definition}

Note that the above operations could also be defined based on rows of the matrix $\bX$. In the rest of the paper, we only refer to Definition~\ref{dfn:vect} and \ref{dfn:mat} for the above operations. Furthermore, in the rest of the paper, we fix $d=p^2$, for the sake of simplicity. We note that there is nothing special about $\bA$ being square and all of our discussion would apply to arbitrary rectangular matrices as well. The advantage of focusing on square matrices is a simplified exposition and reduction in the number of parameters of which we need to keep track. In the case of $d$ not being a perfect square, Definition~\ref{dfn:mat} could be changed to handle rectangular matrices.


\section{Low-Rank Principal Eigenmatrix Analysis} \label{sec:modelmain}

In this section, we first introduce the Low-rank Principal Eigenmatrix Analysis problem in \S\ref{sec:model} and then introduce the matricized rank-truncated power method for solving the problem in \S\ref{sec:MatRankTrunc}.
\subsection{The Model and The Problem}\label{sec:model}
Consider the following noisy ($d \times d$) matrix model:
\begin{equation} \label{eqn:perturbation}
\bA = \bbA + \bE,
\end{equation}
where $\bA$ is a noisy observation of the true signal matrix $\bbA$ with the noise matrix represented by the matrix $\bE$. As a motivating application, we consider $\bA$ to be the empirical covariance matrix and $\bbA$ to the true covariance matrix. But the above model is general and applies to several other situations, for example, $\bA$ could be the noisy version of the true image $\bbA$, represented in term of pixel arrays. The problem of PCA considers estimating the top eigenvector, $\bbx$, of the matrix $\bbA$ based on the observed matrix $\bA$. Sparse PCA assumes $\bbx$ is $s$-sparse and estimates it by following constrained quadratic maximization problem:
\begin{equation*}
\begin{aligned}
& \underset{\bx \in \mathbb{R}^d}{\text{maximize}} & & \bx^\top \bA \bx \\
& \text{subject to} & & \|\bx\|_{\ell_2} =1, \; \|\bx\|_{\ell_0} \leq s.
\end{aligned}
\end{equation*}

Low-rank Principal Eigenmatrix Analysis assumes that the top eigenvector $\bbx$ of $\bA$, when matricized as $\bbX$, according to the Definition~\ref{dfn:mat}, has low-rank. Note that, we are still estimating the top eigenvector. But the low-rank structure is made on the matricized eigenvector $\bbX$. Hence we call it the eigenmatrix, i.e., $\bbX = \mat{\bbx}$ is the top \textit{eigenmatrix} of $\bA$. Furthermore, though $\bbX$ is a low-rank matrix, the eigenvector $\bbx =\vect{\bbX}$ can be dense, different from sparse PCA. Before proceeding, we make the following remarks. First, as discussed in~\S\ref{sec:intro}, our assumption $d=p^2$ is done only for the sake of convenience. In the case of general $d$, one could assume that there exist a $p_1 \times p_2$ rectangular matricization for which such a low-rank structure exists. Next, note that one could assume a more general model, where the vector $\bbx$, after a permutation of its indices followed by the matricization operation has a low-rank. Further, the permutation could be estimated as well. In this work, we do not concentrate on the general model, but we plan to address this in the future. Our goal in this work is to provide a deeper understanding of the  simpler model.

Based on our assumption, the \textit{Low-Rank Principal Eigenmatrix Estimator} is given by the following optimization problem (\ref{eqn:matricized_low_rank}):
\begin{equation} \label{eqn:matricized_low_rank}
\begin{aligned}
& \underset{\bX \in \mathbb{R}^{p \times p}}{\text{maximize}} & & \vect{\bX}^\top \bA \, \vect{\bX} \\
& \text{subject to} & & \|\bX\|_{F} = 1, \\
&&& \rank{\bX}  \leq k\ll p.
\end{aligned}
\end{equation}
Note that the above optimization problem is the natural maximum likelihood solution for estimating the top eigenmatrix $\bbX$ along with the constraints motivated by our assumption.  The above problem is a non-convex problem -- indeed the presence of the rank constraint along with the normalization constraint makes it highly non-convex. In this next section, we propose an iterative method for solving the above optimization problem.


\subsection{\textsc{smart-pm} for Solving Problem~\ref{eqn:matricized_low_rank}} \label{sec:MatRankTrunc}

In general, problem (\ref{eqn:matricized_low_rank}) is non-convex and NP-hard. In order to solve it efficiently, we propose a \textit{Sequentially MATricized Rank Truncated-Power Method (\textsc{smart-pm})} outlined in Algorithm~\ref{alg:mlrpower}. It is an iterative procedure based on the standard power method for eigenvalue problems, while maintaining the desired matricized low-rank structure for the intermediate solutions. 
\begin{definition}
We define Matricized Rank-Truncation operator as following, 
$$RankTrunc(\bX, k) = P_{U} \bX  P_{V}$$ for a matrix $\bX$  
and $$RankTrunc(\bx, k) = \vect{RankTrunc(\mat{\bx}, k)}$$ for a vector $\bx$, where $U$ and $V$ consists of the first $k$ columns of the left and right singular matrices of $\bX$, respectively. $P_{U}$ and $P_{V}$ are rank-$k$ projection matrix onto $U$ and $V$, respectively. For a matrix, operation $RankTrunc(\bX, k)$ is basically the truncated SVD that retains top $k$ singular values and sets the rest to zero.
\end{definition}

Given an initial approximation $\bX_0$ and rank $k$, Algorithm~\ref{alg:mlrpower} generates a sequence of intermediate low-rank eigenmatrices $\bX_1, \bX_2, \ldots$ satisfying $\norm{\bX_t}_{F}=1$ and $\rank{\bX_t} \le k$. At each iteration, the computational complexity is in $O(p^4 + p^2 k)$ which is $O(d^2)=O(p^4)$ for matrix-vector product $\bA \vect{\bX_{t-1}}$ and $O(p^2 k)$ for getting largest $k$ singular values and corresponding singular vectors. The extra computation time $O(p^2k)$ is negligible since $k$ is often small and $O(p^4)$ is often the dominant term. In the next section we analyze the \textsc{smart-pm} algorithm and discuss its rates of convergence.

\begin{algorithm}[th!]
	\caption{\textsc{smart-pm}}
	\label{alg:mlrpower}
	\begin{algorithmic}
		\STATE {\bfseries Input:} $\bA \in \mathbb{R}^{p^2 \times p^2}$ and an initial matrix $\bX_0 \in \mathbb{R}^{p \times p}$. 
		\STATE {\bfseries Output:}  $\bX_t \in \mathbb{R}^{p \times p}$.
		\STATE {\bfseries Parameters:}  Rank of the eigenmatrix $k \in \{1, \dots, p\}$
		Let $t = 1$. 
		\REPEAT
		\STATE Compute $\bX'_t = \mat{\frac{\bA \, \vect{\bX_{t-1}}}{\|\bA \, \vect{\bX_{t-1}}\|_{\ell_2}}}$
		\STATE Compute $\bX_t = RankTrunc(\bX'_t, k)$.
		\STATE Normalize $\bX_t = \frac{\bX_t}{\|\bX_t\|_{F}}$.
		\STATE $t = t+1$
		\UNTIL{Convergence.}
	\end{algorithmic}
\end{algorithm}

%% file: TPM.tex
\section{Theoretical Analysis of \textsc{smart-pm}}
In this section, we provide a computational and statistical convergence analysis of the \textsc{smart-pm} algorithm. Our first result is on the algorithmic convergence of  \textsc{smart-pm}.  Specifically, we show that when all the rank $4k^2$ projections of the matrix $\bA$ are positive semi-definite, the iterates of the \textsc{smart-pm} algorithm are monotonically increasing in terms of the objective value it is optimizing. 
\begin{proposition}
If all the rank $4k^2$ projections of the matrix $\bA$ are positive semi-definite, then the sequence $\{ Q( \vect{\bX_t} ) \}_{t \ge 1}$ is a monotonically increasing sequence, where $\bX_t$ is obtained from the SMART-PM algorithm.  
\end{proposition}
\begin{proof}
Note that the iterate $\bx_t = \vect{\bX_t}$ in the \textsc{smart-pm} algorithm solves the following constrained optimization problem:
\begin{equation*} 
\begin{aligned}
& \underset{\bx}{\text{maximize}} & & L(\bx, \bx_{t-1}) \coloneqq \langle 2 \bA \bx_{t-1}, \bx - \bx_{t-1} \rangle  \\
& \text{subject to} & & \|\bx\|_{\ell_2} = 1, \\
&&& \rank{\bX}  \leq k < p.
\end{aligned}
\end{equation*}
For any $\textsc{rank}(\bX) \leq k$ and $\textsc{rank}(\bX_{t-1}) \leq k$, we have $\textsc{rank}(\bX - \bX_{t-1}) \leq 2k$. Suppose $\bX - \bX_{t-1}$ assumes the following form of SVD decomposition $\bX - \bX_{t-1} = \bU \bD_{2k} \bV^\top$ where $\bU$ and $\bV$ are left and right singular matrix, respectively, then $\bX - \bX_{t-1} = P_U (\bX - \bX_{t-1}) P_V$ and $\bx - \bx_{t-1} = P_{V \otimes U} (\bx - \bx_{t-1})$, where $P_U$ is the projection matrix onto the column space of matrix $\bU$. Since all the rank $4k^2$ projections of the matrix $\bA$ are positive semi-definite, we have 
\begin{align*}
& (\bx_t - \bx_{t-1})^T \, \bA \, (\bx_t - \bx_{t-1}) \\
& = (\bx_t - \bx_{t-1})^T \, P_{V \otimes U} \bA P_{V \otimes U} \, (\bx_t - \bx_{t-1}) \\
& \ge 0. 
\end{align*}
Clearly, 
\begin{align*}
&Q(\vect{\bX_t}) - Q(\vect{\bX_{t-1}})  \\
& = L(\bx_t, \bx_{t-1}) + (\bx_t - \bx_{t-1})^T \, \bA \, (\bx_t - \bx_{t-1}) \\
& \ge  L(\bx_t, \bx_{t-1}) \ge L(\bx_{t-1}, \bx_{t-1}) \\
& = 0.
\end{align*}
The second inequality holds because $\bx_t$ by definition maximizes $L(\bx, \bx_{t-1})$ at step $t$.
\end{proof}


The above results is a purely computational result justifying the use of \textsc{smart-pm} method as a heuristic for solving the non-convex problem in Equation~\ref{eqn:matricized_low_rank}. We next consider the general noisy matrix perturbation model~(\ref{eqn:perturbation}) and show that if matrix $\bbA$ has a unique low-rank (or approximately low-rank) top eigenmatrix, then under suitable condition, \textsc{smart-pm} can 
estimate this eigenmatrix from the noisy observation $\bA$, under certain conditions on the initial matrix $\bX_0$. In order to establish such a result, we first list the set of assumptions required, precisely. 

\begin{assumption} \label{assump:matricized_low_rank}
Assume that the largest eigenvalue of $\bbA \in \mathbb{S}^{p^2}$ is $\lambda = \lambda_{\max}(\bbA) > 0$ that is non-degenerate, with a gap $\triangle \lambda = \lambda - \max_{j>1} |\lambda_j(\bbA)|$ between the largest and the remaining eigenvalues. Moreover, assume that the eigenvector $\bar\bx \in \mathbb{R}^d$ corresponding to the dominant eigenvalue $\lambda$ is of rank $k$ when matricized, i.e. $\mat{\bar\bx} \in \mathbb{R}^{p \times p}$ is of rank $k \le p$.
\end{assumption}

\begin{assumption} \label{assump:E_no_matricized_low_rank}
Assume that $\bE \in \mathbb{S}^{p^2 \times p^2}$ does not have low-rank eigenmatrices. Mathematically, let $\bw$ be any eigenvector of $\bE$, $\|\bw\|_{\ell_2}^2 = 1$, and $\bW = \mat{\bw} \in \mathbb{R}^{p \times p}$ be the corresponding eigenmatrix. $\bW$ is of rank $p$ and $\max\{ |\sigma_1(\bW)|, |\sigma_p(\bW)| \} \asymp \sqrt{1/p}$.   
\end{assumption}
\begin{remark}
Assumption~\ref{assump:matricized_low_rank} posits that the top eigenvector (or the eigenmatrix) of $\bbA$ is isolated. Such an analysis is common in the analysis of power method. Assumption~\ref{assump:E_no_matricized_low_rank} is crucial and captures a particular conditions on the error matrix $\bE$ under which the~\textsc{smart-pm} has efficient estimation rates. It posits that the eigenmatrices of $\bE$ have are high-rank with a flat spectrum.  This assumption is complementary to the structural assumption made on the true signal to be estimated, i.e., top eigenmatrix $\bbX$ is low-rank. For example, if the true signal matrix $\bbA$ is corrupted by an additive standard Gaussian iid noise matrix $\bE$, then with high probability such an assumption holds naturally.
\end{remark}


Now we first show that, under Assumption \ref{assump:E_no_matricized_low_rank} when the noise matrix $\bE$ does not have low-rank eigenmatrices, the spectral norm of the low rank projection $\rho(P_{V \otimes U} \bE P_{V \otimes U})$ can be small, even when $\rho(\bE)$ is large. 
\begin{lemma} \label{lemma:E_low_rank_proj}
Consider low-rank projection $P_U$ and $P_V$ where $\rank{P_U} = \rank{P_V} = k$ and $1 \le k \le p$. For any matrix $\bE \in \mathbb{S}^{p^2 \times p^2}$ satisfying Assumption \ref{assump:E_no_matricized_low_rank}, $\rho(\bE_{UV}) = O(\sqrt{k/p}) \cdot \rho \left( \bE \right)$, where $\bE_{UV}=P_{V \otimes U} \bE P_{V \otimes U}$. 
\end{lemma}
\begin{proof}
Suppose $\bE = \sum_{i=1}^{p^2} \lambda_i \bw_i \bw_i^\top$ and $\norm{\bw_i}^2 = 1$. For $1 \le i \le p^2$, let $\bW_i = \mat{\bw_i}$ be the $i$-th eigenmatrices of $\bE$. Then 
\begin{align*}
& \norm[\Big]{P_U \bW_i P_V^\top}_F \le \sqrt{\rank{P_U \bW_i P_V^\top}} \, \norm[\Big]{P_U \bW_i P_V^\top} \\
& \qquad \le \sqrt{k} \norm[\Big]{\bW_i}  = O(\sqrt{k/p}).
\end{align*}
Then $\norm{P_{V \otimes U} \bw_i}_{\ell_2} = O(\sqrt{k/p})$. 

Let $\bW = \begin{pmatrix} \bw_1, \cdots, \bw_{p^2}\end{pmatrix}$, $\bE = \bW \bLambda \bW^\top$ and $\bQ = P_{V \otimes U} \bW$, we have $\norm{\bQ}_F^2= \norm{\bQ \bQ^\top}_\star = O(kp)$, $\norm{\bQ}^2 = O(k/p)$ and $P_{ V\otimes U} \bE \bW = \bQ \bLambda$. Thus,
\begin{align*}
& \rho^2(\bE_{UV}) = \norm{\bE_{UV}}^2 = \norm{\bQ \bLambda \bQ^\top \bQ \bLambda \bQ^\top} \\
& = \norm{\bQ \bLambda \bQ^\top \bQ \bLambda \bQ^\top} \\
& \le \norm{ \bLambda \bQ^\top \bQ \bLambda } \norm{ \bQ }^2 \\
& = \norm{ P_{ V\otimes U} \bE \bW }^2 \norm{ \bQ }^2 \\
& = O(k/p) \cdot \rho^2(\bE).
\end{align*}
Thus, we have 
\begin{align*}
\rho \left(\bE_{UV} \right) = O(\sqrt{k/p}) \cdot \rho \left( \bE \right).
\end{align*}
\end{proof}


We now state our main result as below, which shows that under appropriate conditions, the \textsc{smart-pm} algorithm can recover the low rank eigenmatrix. 

\begin{theorem} \label{thm:SMART-PM}
Assume that Assumption \ref{assump:matricized_low_rank} and \ref{assump:E_no_matricized_low_rank} hold. Let $1 \le \bar k \le k$. Assume that $\kappa \coloneqq \triangle \lambda / \rho(\bE_{UV}) > 2$ where $\bE_{UV}$ is defined in Lemma \ref{lemma:E_low_rank_proj} for any rank-$k$ matrices $U$ and $V$. Define
\begin{align} \label{eqn:eigen_gap_A}
\gamma \coloneqq \frac{\lambda - \triangle \lambda + \rho \left( \bE_{UV} \right)}{\lambda - \rho \left( \bE_{UV} \right)},
\end{align}
and 
\begin{align}  \label{eqn:err_from_sig_noise}
\delta \coloneqq \frac{\sqrt{2}}{\sqrt{1 + \left(\kappa - 2 \right)^2}}.
\end{align}
If $|\langle \bX_0,  \bbX \rangle| \ge \theta + \delta$ for some $\bX_0$ with $\rank{\bX_0}=k$ and $\norm{\bX_0}_F=1$ and $\theta \in (0,1)$ such that
\begin{equation*}
\resizebox{\hsize}{!}{$\mu \coloneqq \sqrt{1 + 2((\bar k / k)^{1/2} + \bar k / k)} \sqrt{1- 0.5 \theta (1+\theta) (1-\gamma^2)} < 1.$}
\end{equation*}
Then we either have
\begin{equation}
\sqrt{1 - |\langle \bX_0,  \bbX \rangle|} \le \sqrt{10} \delta /(1-\mu)
\end{equation}
or for all $t \ge 0$
\begin{equation}
\sqrt{1 - |\langle \bX_t,  \bbX \rangle|} \le \mu^t \sqrt{1 - |\langle \bX_0,  \bbX \rangle|} + \sqrt{10} \delta /(1-\mu).
\end{equation}
\end{theorem}

\begin{proof}
We sketch the proof here, while the details are relegated to Appendix \ref{appdx:proof_SMART-PM} in supplementary material. The proof is carried out in the following steps. 
\begin{enumerate}
\item Lemma \ref{lemma:perturbation_rank_truncated} establishes the perturbation theory of symmetric eigenvalue problem for rank truncated $\bA_{UV} = \bbA_{UV} + \bE_{UV}$, i.e. bounds $\norm{\bx_1(\bA_{UV}) - \bbx}_{\ell_2}$, where $\bbA_{UV} = \lambda_{1} \bbx \bbx^\top$ and $\bx_1(\bA_{UV})$ is the eigenvector corresponding to the largest eigenvalue of $\bA_{UV}$.
\item Lemma \ref{lemma:rank_trunc_err} quantifies the error introduced by the rank-truncation step in the SMART-PM, i.e. $|\langle RankTrunc(\bX'_t, k),  \bbX \rangle| - |\langle \bX', \bbX \rangle |$.
\item Lemma \ref{lemma:MatRankTrun_iter_better} establishes that each step of the SMART-PM improves eigenvector estimation, i.e. $\sqrt{1 - |\langle \bX_t,  \bbX \rangle|}$ decreases geometrically with factor $\mu$.
\item Based on results from Lemma \ref{lemma:perturbation_rank_truncated}, \ref{lemma:rank_trunc_err}, and \ref{lemma:MatRankTrun_iter_better}, this step is similar to that of Theorem 4 in \cite{yuan2013truncated} and thus omitted in the Appendix \ref{appdx:proof_SMART-PM}.
\end{enumerate}
\vspace{-0.1in}
\end{proof}
\begin{remark}
For any fixed eigen-gap $\triangle \lambda$, if $\rho(\bE_{UV}) < \triangle \lambda / 2$, then $\gamma <  1$ and $\delta = O(1/\kappa) = O(\rho(\bE_{UV}))$. If $\gamma$ is sufficiently small, then the requirement that $\mu < 1$ can be satisfied for a sufficiently small $\theta$ of the order $\sqrt{\bar k / k}$. Theorem \ref{thm:SMART-PM} shows that under appropriate conditions, as long as we can find an initial $\bX_0$ such that 
$$| \langle \bX_0, \bbX \rangle | \ge C \left( \rho(\bE_{UV}) + \sqrt{\bar k / k} \right)$$ for some constant $C$, then $\sqrt{1 - |\langle \bX_t,  \bbX \rangle|}$ converges geometrically until 
$$\| \bX_t - \bbX \|_F = O \left(\rho(\bE_{UV})\right) = O \left(\sqrt{k/p} \, \rho(\bE) \right).$$
\end{remark}
Theorem \ref{thm:SMART-PM} thus provides strong theoretical justification of the proposed \textsc{smart-pm} algorithm. Specifically, the replacement of the full matrix perturbation error $\rho(\bE)$ with $\sqrt{k/p} \, \rho(\bE)$ gives theoretical insights on why SMART-PM gives superior results in Section \ref{sec:empirical}. To illustrate this further, we briefly describe a consequence of Theorem \ref{thm:SMART-PM} in the case, when $\bbA$ is corrupted by an iid standard Gaussian noise matrix $\bE$. In this case, $\rho(\bE) = O(\sqrt{p^2})$ with high-probability. The spectral norm of the low rank projection $\rho(P_{V \otimes U} \bE P_{V \otimes U}) = O(\sqrt{pk})$, which grows linearly in $\sqrt{pk}$ instead of $p$. Hence, we can run \textsc{smart-pm} with an appropriate initial vector to obtain an approximate solution $\bX_t$ with error $\| \bX_t - \bbX \|_F = O \left( \sqrt{pk} \right)$. Note that for estimating a $p \times p$ matrix with rank $k$, a simple parameter counting argument shows that the minimax lower bound is of the order $O \left( \sqrt{pk} \right)$. Thus the \textsc{smart-pm} algorithm achieves the minimax lower bound provided an initializer satisfying the conditions of Theorem~\ref{thm:SMART-PM} is used. The question of obtaining such an initializer is more delicate. In the case of sparse PCA, an interesting statistical-computational trade-off exists; see for example~\cite{berthet2013optimal, ma2015sum, ma2015computational, wang2016statistical, brennan2018reducibility}. We conjecture that such a phenomenon exists in our problem as well.

Finally, we remark that while the steps of the proof of Theorem~\ref{thm:SMART-PM} are motivated by the proof of Theorem 4 in \cite{yuan2013truncated} on the analysis of truncated power method for sparse PCA, the details of our proof are very different because of our matricized low-rank structure on the eigenvectors and the involved non-commutativity issues involved. Also, note that Lemma \ref{lemma:E_low_rank_proj} qualifies the largest absolute eigenvalue of a low rank projection of noise matrix $\rho \left( \bE_{UV} \right) = O(\sqrt{k/p}) \cdot \rho \left( \bE \right)$, which gives a different rate from that of the sparse case.

%% file: empirical.tex
\section{Experiments} \label{sec:empirical}

In this section, we show simulation results that confirm the relevance of our theoretical findings regarding \textsc{smart-pm} algorithm. The data $\{\by_1, \ldots, \by_n\}$ are generated from $\mathcal{N}(0, \bbA)$ with true covariance 
$$\bbA = \lambda_1 \bbx \bbx^\top + \bSigma_{\epsilon},$$
where $\norm{\bSigma_{\epsilon}} = \lambda_2$ is set to $1$ throughout experiments. The empirical covariance is 
$$\bA = \frac{1}{n} \sum_{i=1}^{n} \by_i \by_i^\top.$$ 
Theorem \ref{thm:SMART-PM} implies that under appropriate conditions, the estimation error $\| \bX_t - \bbX \|_F$ is proportional to $\delta$ which is an increasing function of $\rho(\bE_{UV})$ and decreasing function of eigen-gap $\triangle \lambda = \lambda_1 - \lambda_2$. For this model, we have $\rho(\bE_{UV}) = O(\sqrt{pk/n})$. Therefore, for fixed dimensionality $p^2$, the error bound is controlled by the triplet $\{n, k, \lambda_1 - \lambda_2\}$. 

In this study, we consider a setup with dimensionality $ d = p^2 = 32^2 = 1024$, and the eigenmatrix $\bbX = \mat{\bbx}$ a rank $\bar k$ uniform random matrix with $\bar k = 1$ and $\norm{\mat{\bbX}}_F = 1$. We compare the convergence trajectory, sample efficiency and effect of eigen-gap between \textsc{smart-pm} and vanilla power method. We also show how different value of $k$ affects the performance of \textsc{smart-pm}. 

\subsection{Convergence Trajectory}

Figure \ref{fig:trajectory_max_eigen_5}, \ref{fig:trajectory_max_eigen_10} and \ref{fig:trajectory_max_eigen_100} show the convergence trajectory for \textsc{smart-pm} and Power method for $n=800$, $k=2$, and different $\lambda_1 \in \{5, 10, 100\}$. The y-axis corresponds to the log of estimation error $\| \bX_t - \bbX \|_F$ and the x-axis corresponds to the number of iterations. We tried three different methods of initialization: (1) Randomly generated $\bbX_0$; (2) Rank-$k$ randomly generated $\bbX_0$; (3) Maximum eigenmatrix of $\bA$. It can be seen from Figure \ref{fig:trajectory} that all methods converge faster when the eigen-gap is larger. \textsc{smart-pm} with maximum eigenmatrix of $\bA$ as initial value (\textsc{smart-pm} 3) converges faster than the other methods. \textsc{smart-pm} with random initialization (\textsc{smart-pm} 1)  converge faster than simple power method under all settings. However, random rank-$k$ initialization (\textsc{smart-pm} 2) does not work as well as random initialization (\textsc{smart-pm} 1) perhaps due to $k$ being set at 2, while $\bar{k} =1$.
%
 
\begin{figure*}[th]  
\centering
\subfigure[$\lambda_1 = 5$, $\lambda_2=1$]{\label{fig:trajectory_max_eigen_5}\includegraphics[scale=0.687]{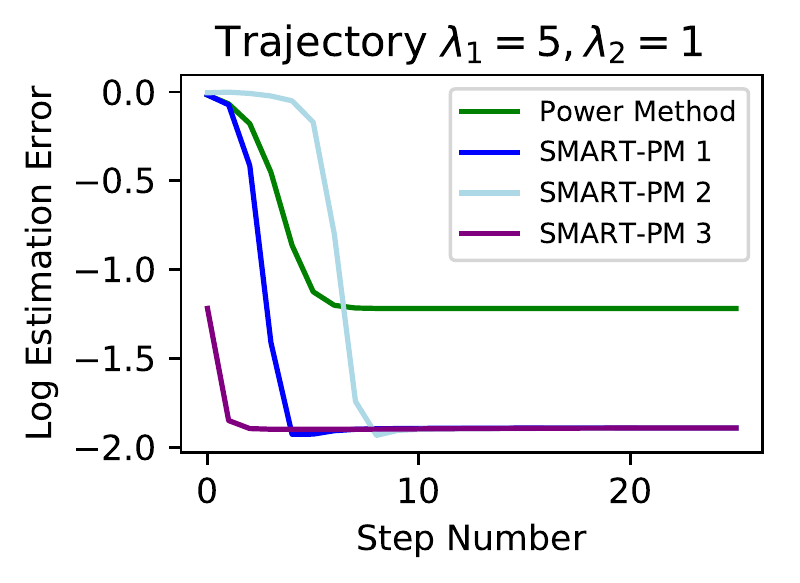}}
\subfigure[$\lambda_1 = 10$, $\lambda_2=1$]{\label{fig:trajectory_max_eigen_10}\includegraphics[scale=0.687]{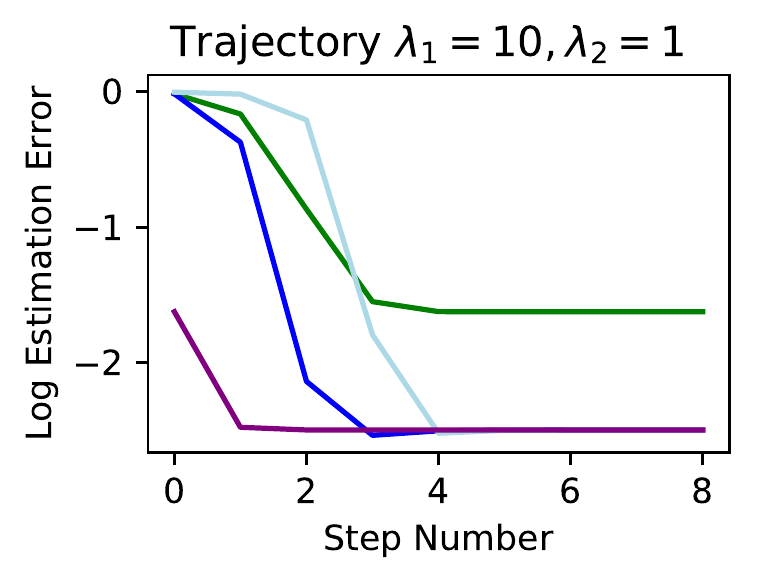}}
\subfigure[$\lambda_1 = 100$, $\lambda_2=1$]{\label{fig:trajectory_max_eigen_100}\includegraphics[scale=0.687]{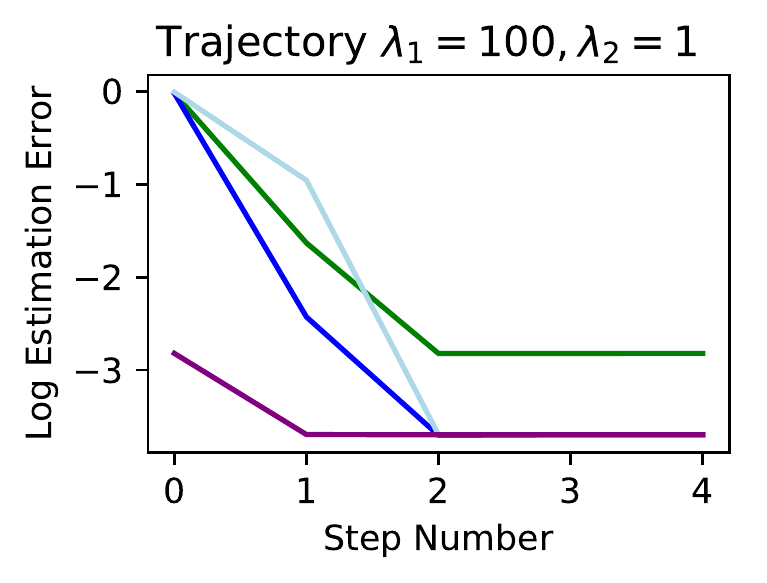}}
\caption{Convergence Trajectory of \textsc{smart-pm} and Power method. \textsc{smart-pm} 1,2,3 represent three different methods to initialize $\bX_0$: (1) Randomly generated; (2) Truncated to rank-$k$ after randomly generated; (3) Maximum eigenmatrix of $\bA$.}  \label{fig:trajectory}
\end{figure*}

\subsection{Sample Efficiency and Eigen-Gap Effect}

In this case, we work with $n \in \{100, 200, 400, 800, 1600\}$, $\lambda_1 \in \{5, 10, 100\}$ and $\lambda_2=1$ given $k = 2 \ge \bar k$. $k$ can also be tuned and selected using cross-validation as in \cite{yuan2013truncated}. For each pair $\{n, \lambda_1\}$, we generate $100$ data sets, calculate the empirical covariance matrices and employ \textsc{smart-pm} and Power method to compute a rank-$k$ eigenmatrix $\hat\bX$.  

Figure \ref{fig:max_eigen_5}, \ref{fig:max_eigen_10} and \ref{fig:max_eigen_100} shows the log of estimation curves as functions of sample size $n$ increases with $\lambda_1 \in \{5, 10, 100\}$, respectively. For fixed eigen-gap, the estimation error decreases as sample size $n$ increases. For fixed sample size, the estimation error is smaller with larger eigen-gap. Again, estimation error by \textsc{smart-pm} 3 with maximum eigenmatrix of $\bA$ as initiate $\bX_0$ has smaller mean and variance. \textsc{smart-pm} 1 with random initialization performs as well as \textsc{smart-pm} 3 under almost all setting except in the beginning of Figure \ref{fig:max_eigen_5} when both sample size and eigen-gap are very small. \textsc{smart-pm} 2 with rank-$k$ random initialization performs well when sample size or eigen-gap is large enough but may gets unstable with small eigen-gap and sample size. In all settings, \textsc{smart-pm} outperforms the Power method. 

\begin{figure*}[th] 
	\centering
	\subfigure[$\lambda_1 = 5$, $\lambda_2=1$]{\label{fig:max_eigen_5}\includegraphics[scale=0.68]{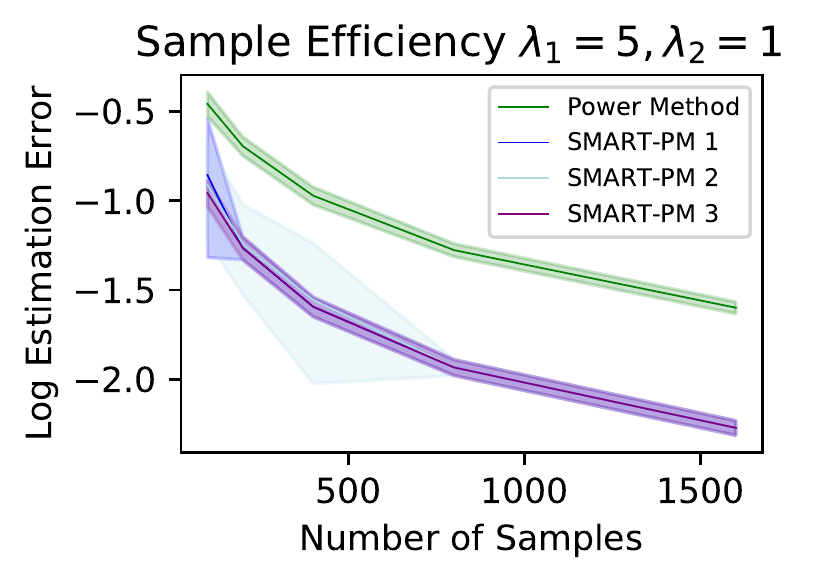}}
	\subfigure[$\lambda_1 = 10$, $\lambda_2=1$]{\label{fig:max_eigen_10}\includegraphics[scale=0.68]{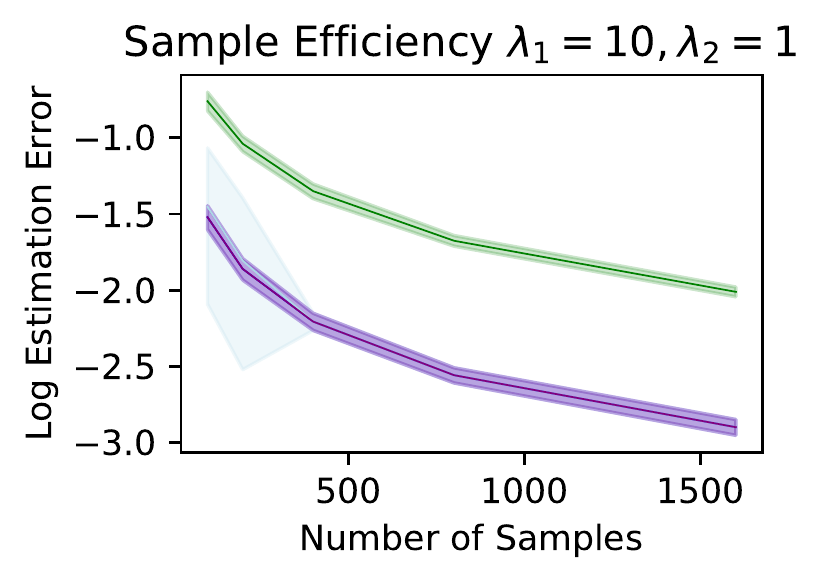}}
	\subfigure[$\lambda_1 = 100$, $\lambda_2=1$]{\label{fig:max_eigen_100}\includegraphics[scale=0.68]{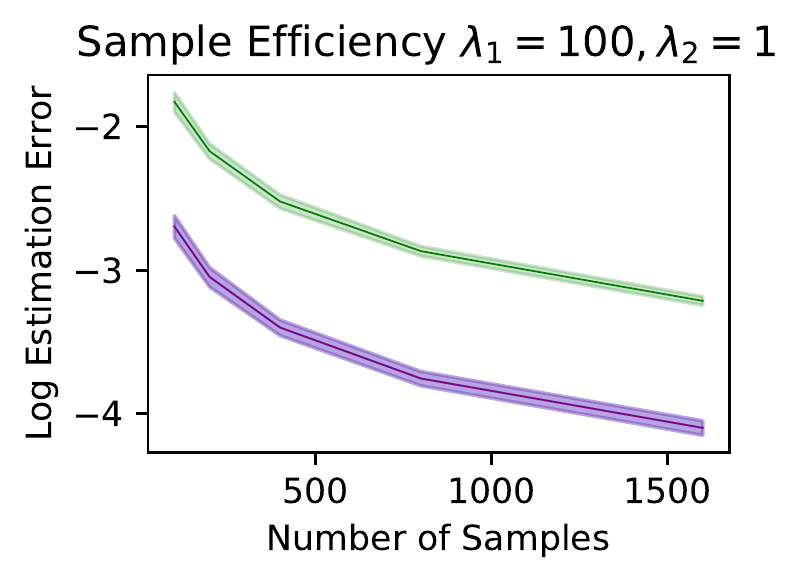}}
	\caption{Sample efficiency and eigen-gap effect of \textsc{smart-pm} and Power method from $100$ simulation runs. The solid line represents the mean error, while the shaded area around denotes the standard deviation. \textsc{smart-pm} 1,2,3 represent three different methods to initialize $\bX_0$: (1) Randomly generated; (2) Truncated to rank-$k$ after randomly generated; (3) Maximum eigenmatrix of $\bA$.} \label{fig:max_eigen}
\end{figure*}

\subsection{Varying Input Rank $k$}

In this case, we fix sample size $n=100$, $\lambda_1 = 100$ and $\lambda_2=1$, and test the values of $ k \in \{1,2,4,8,16,32\}$. We generate $100$ empirical covariance matrices and employ \textsc{smart-pm} and power method to compute a rank-$k$ eigenmatrix. From experiments in the previous sections, the random initialization of $\bX_0$ performs well. Hence, we employ the random initialization here. Recall that $p=32$ in our setting. When $k=p=32$, \textsc{smart-pm} degenerates to the power method. Figure \ref{fig:k_study_100} shows the log of estimation error curve as function of $k$. The power method is not relevant to $k$ so the corresponding curve is a horizontal line. The error curve of \textsc{smart-pm} stays below while approaching that of the power method as $k$ increases. Also, the variance of the error becomes larger as input rank $k$ increases. Similar observations are also made for other fixed pairs $\{n, \lambda_2 - \lambda_1\}$.

\begin{figure}[t!]
	\centering	
	\includegraphics[width=0.8\linewidth, keepaspectratio]{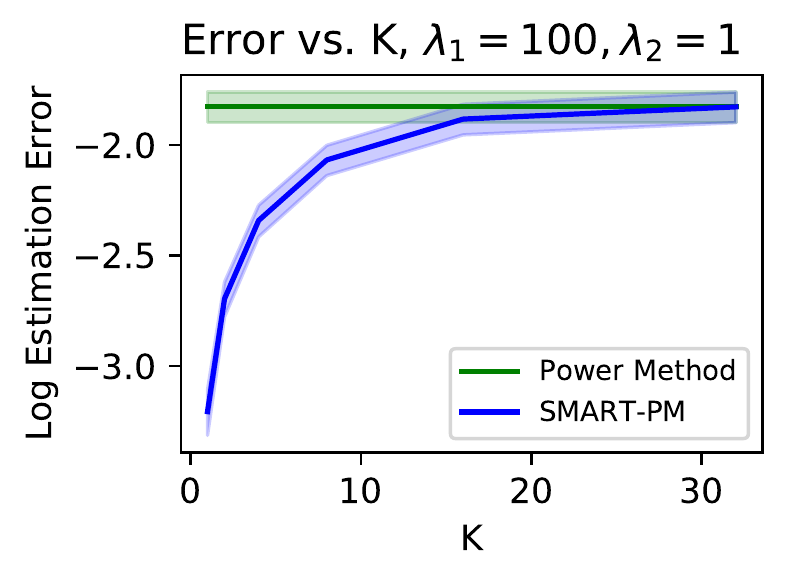}
	\caption{Effect of input rank $k$ from $100$ simulation runs. The solid line represents the mean error, while the shaded area around denotes the standard deviation. Use Randomly generated $\bX_0$ as initialization.}
	\label{fig:k_study_100}
\end{figure}

%% file: appendix.tex
\begin{appendices}
\section{Proof of Theorem \ref{thm:SMART-PM}} \label{appdx:proof_SMART-PM}

This section provides detailed proof of Theorem \ref{thm:SMART-PM}. It is carried out in the following steps: Lemma \ref{lemma:perturbation_rank_truncated} establishes the perturbation theory of symmetric eigenvalue problem for rank truncated $\bA_{UV} = \bbA_{UV} + \bE_{UV}$, i.e. bounds $\norm{\bx_1(\bA_{UV}) - \bbx}_{\ell_2}$, where $\bbA_{UV} = \lambda_{1} \bbx \bbx^\top$ and $\bx_1(\bA_{UV})$ is the eigenvector corresponding to the largest eigenvalue of $\bA_{UV}$. Lemma \ref{lemma:convg_of_power_method} conducts convergence analysis of traditional power method. Lemma \ref{lemma:rank_trunc_err} quantifies the error introduced by the rank-truncation step in the \textsc{smart-pm}, i.e. $|\langle RankTrunc(\bX'_t, k),  \bbX \rangle| - |\langle \bX', \bbX \rangle |$. Lemma \ref{lemma:MatRankTrun_iter_better} establishes that each step of the \textsc{smart-pm} improves eigenvector estimation, i.e. $\sqrt{1 - |\langle \bX_t,  \bbX \rangle|}$ decreases geometrically with factor $\mu$. In the proof, we denote the pair of bold-faced $(\bx_t, \bX_t)$ as the pair of vectorization and matricization and use frequently the fact that $\bx_t^\top \bbx = \langle \bX_t, \bbX \rangle$.

\subsection{The Perturbation Theorem of  Rank-Truncated Symmetric Eigenvalue Problem}

In this section, main lemma \ref{lemma:perturbation_rank_truncated} states the perturbation theory of symmetric eigenvalue problem for rank truncated $\bA = \bbA + \bE$. 
\begin{lemma}  \label{lemma:perturbation_rank_truncated}
	Consider projection $P_U$ and $P_V$ with $P_U \bbX P_V = \bbX$ where $\rank{P_U} = \rank{P_V} = k$ and $1 \le k \le p$. Under Assumption \ref{assump:matricized_low_rank} and \ref{assump:E_no_matricized_low_rank}, if $\kappa \coloneqq \triangle \lambda / \rho(\bE_{UV}) > 2$, then the ratio of the second largest (in absolute value) to the largest eigenvalue of matrix $\bA_{UV} $ is no more than 
	\begin{align*} 
	\gamma \coloneqq \frac{\lambda - \triangle \lambda + \rho \left( \bE_{UV} \right)}{\lambda - \rho \left( \bE_{UV} \right)}.
	\end{align*}
	Moreover,
	\begin{align*}  
	\norm{\bx_1(\bA_{UV}) - \bbx}_{\ell_2} \le \delta \coloneqq \frac{\sqrt{2}}{\sqrt{1 + \left(\kappa - 2 \right)^2}}.
	\end{align*}
\end{lemma}
\begin{proof}
	Define $\bA_{UV} \coloneqq P_{V \otimes U} \bA P_{V \otimes U}$, $\bbA_{UV} \coloneqq P_{V \otimes U} \bbA P_{V \otimes U}$, and $\bE_{UV} \coloneqq P_{V \otimes U} \bE P_{V \otimes U}$. Using Weyl's inequality, we obtain
	\begin{align*}
	&\lambda_1(\bA_{UV}) \ge \lambda_1(\bbA_{UV}) + \lambda_p(\bE_{UV}) \ge \lambda - \rho(\bE_{UV}) 
	\end{align*}
	and $\forall j \ge 2$, 
	\begin{align*}
	\lambda_j(\bA_{UV}) &\le \lambda_j(\bbA_{UV}) +\lambda_1(\bE_{UV}) \\
	&   \le \lambda - \triangle \lambda + \lambda_{1}(\bE_{UV}) \\ &\le \lambda - \triangle \lambda + \rho(\bE_{UV}).
	\end{align*}
	Then, 
	\begin{align*}
	\frac{\lambda_j(\bA_{UV})}{\lambda_1(\bA_{UV})} \le \frac{\lambda - \triangle \lambda + \rho \left( \bE_{UV} \right)}{\lambda - \rho \left( \bE_{UV} \right)}.
	\end{align*}
	Now write the eigenvector w.r.t the largest eigenvalue $\lambda_1(\bA_{UV})$ as $\bx_1(\bA_{UV}) = \alpha \bbx + \beta \bz$, where $\norm{\bbx}_{\ell_2} = \norm{\bz}_{\ell_2} = 1$, $\bbx^\top \bz = 0$ and $\alpha^2 + \beta^2 = 1$. This implies that
	$$\bA_{UV} \left( \alpha \bbx + \beta \bz \right) = \lambda_1(\bA_{UV}) \left( \alpha \bbx + \beta \bz \right).$$
	Then right multiplying the above equation by $\bz^\top$, we have
	\begin{equation*}
	\alpha \bz^\top \bA_{UV} \bbx + \beta \bz^\top \bA_{UV} \bz = \lambda_1(\bA_{UV}) \beta.
	\end{equation*}
	This leads to,
	\begin{align*}
	|\beta| & = \left| \frac{\alpha \bz^\top \bA_{UV} \bbx}{\lambda_1(\bA_{UV}) - \bz^\top \bA_{UV} \bz} \right| \\
	& \le | \alpha | \left| \frac{ \bz^\top \bA_{UV} \bbx }{\lambda_1(\bA_{UV}) - \bz^\top \bA_{UV} \bz} \right| \\
	& = | \alpha | \left| \frac{ \bz^\top \bE_{UV} \bbx }{\lambda_1(\bA_{UV}) - \bz^\top \bA_{UV} \bz} \right| \\
	& \le | \alpha | t,
	\end{align*}
	where $t \coloneqq {\rho \left( \bE_{UV} \right)}/{(\triangle \lambda - 2\rho \left( \bE_{UV} \right))}$ under the condition that $\triangle \lambda > 2 \rho \left( \bE_{UV} \right)$.
	Then we have $1 = \alpha^2 + \beta^2 \le \alpha^2 (1 + t^2)$ and $\alpha^2 \ge \frac{1}{1 + t^2}$. Without loss of generality, we may assume that $\alpha > 0$ because otherwise we can replace $\bbx$ with $- \bbx$. 
	It follows that 
	\begin{align*}
	\norm{\bx_1(\bA_{UV}) - \bbx}^2_{\ell_2} & = 2 - 2 \bx_1(\bA_{UV})^\top \bbx = 2 - 2 \alpha \\
	& \le 2 \frac{\sqrt{1 + t^2} - 1}{\sqrt{1 + t^2}}  \le \frac{2 t^2}{1 + t^2}
	\end{align*}
	We have
	\begin{equation}
	\norm{\bx_1(\bA_{UV}) - \bbx}_{\ell_2} \le \frac{\sqrt{2}}{\sqrt{1 + \left(\kappa - 2 \right)^2}} 
	\end{equation}
\end{proof}

\subsection{The Error Analysis of SMART-PM}
The following lemma from Lemma 11 in \cite{yuan2013truncated} conducts convergence analysis of traditional power method. We restate it here for completeness. 

\begin{lemma} \label{lemma:convg_of_power_method}
	Let $\bbx$ be the eigenvector with the largest (in absolute value) eigenvalue of a symmetric matrix $\bA$, and let $\gamma < 1$ be the ratio of the second largest to largest eigenvalue in absolute values. Given any $\by$ such that $\norm{\by}_{\ell_2}=1$ and $\bbx^\top \by > 0$; Let $\by' = \bA \by / \norm{\bA \by}_{\ell_2}$, then
	\begin{equation*}
	|\bbx^\top \by'| \ge |\bbx^\top \by| [1 + (1-\gamma^2)(1-(\bbx^\top \by)^2)/2].
	\end{equation*}
\end{lemma}

The following lemma quantifies the error introduced by the rank-truncation step in \textsc{smart-pm} algorithm.

\begin{lemma} \label{lemma:rank_trunc_err}
	Consider $\bar\bx$ with $\rank{\mat{\bbx}} = \bar k$. For $\by \in \mathbb{R}^{p^2}$ and $1 \le k < p$, let the singular value decomposition of $\bY = \mat{\by}$ be $\bU \bD \bV^T$ where the diagonal of $\bD$ contains singular values in (absolute) non-increasing order. Define Matricized Rank-Truncation operator $RankTrunc(\mat{\by}, k) = P_{U_k} \bY  P_{V_k}$ and $RankTrunc(\by, k) = \vect{RankTrunc(\mat{\by}, k)}$, where $P_{U_k}$ and $P_{V_k}$ are projection matrix onto the first k columns of $\bU$ and $\bV$,respectively. If $\|\bar\bx\|_{\ell_2} = \|\by\|_{\ell_2} = 1$, then
	\begin{align*}
	&|RankTrunc(\by, k)^\top \bar\bx| \ge | \by^\top \bar\bx |\\
	&- (\bar k/k)^{-1/2} \min \left[\sqrt{1-(\by^\top \bar\bx)^2}, (1+(\bar k / k)^{1/2}) \left(1-(\by^\top \bar\bx)^2 \right) \right].
	\end{align*}
\end{lemma}

\begin{proof}
	Define $\bY = \mat{\by}$, $\by = \vect{\bY}$ and similar for $\bar\bx$ and $\bar \bX$. Suppose the SVD decomposition of $\bY$ ($\bar \bX$) assume the form $\bU \bD \bV^\top$ ($\bP \bLambda \bQ^\top$) where $\bU$, $\bV$, $\bP$ and $\bQ$ are orthonormal matrices and $\bD = \diag{d_1, \ldots, d_p}$ and $\bLambda = \diag{\lambda_1, \ldots, \lambda_k^\star}$ with singular values in non-increasing order. For simplicy, we work with non-negative singular value because otherwise we may just change the sign of one singular vector. Then,  
	\begin{align*}
	\bY  = \sum_{i=1}^{p} d_i \, \bu_i \bv_i^\top, \quad& \by = \sum_{i=1}^{p} d_i \, \bv_i \otimes \bu_i, \\
	\bar \bX  = \sum_{i=1}^{\bar k} \lambda_i \, \bp_i \bq_i^\top, \quad & \bar\bx  = \sum_{i=1}^{\bar k} \lambda_i \, \bp_i \otimes \bq_i.
	\end{align*}
We also have,
	\begin{align*}
	RankTrunc(\bY, k) & = \sum_{i=1}^{k} d_i \bu_i \bv_i^\top = P_{U_k} \bY P_{V_k}^\top,\\
	RankTrunc(\by, k) & = \sum_{i=1}^{k} d_i \bv_i \otimes \bu_i = P_{V_k \otimes U_k} \by, \\
	RankTrunc(\by, k) & = \vect{RankTrunc(\bY, k)}, \\
	RankTrunc(\bY, k) & = \mat{RankTrunc(\by, k)}.
	\end{align*}
	Define $P_1 = P_{P_{\bar k} \otimes Q_{\bar k}} P_{V_k \otimes U_k}^\perp$, $P_2 = P_{P_{\bar k} \otimes Q_{\bar k}} P_{V_k \otimes U_k}$, $P_3 = P_{P_{\bar k} \otimes Q_{\bar k}}^\perp P_{V_k \otimes U_k}$ and $P_4 = P_{P_{\bar k} \otimes Q_{\bar k}}^\perp P_{V_k \otimes U_k}^\perp$. Let $k_i = \rank{P_i}$. Since $\|\bar\bx\|_{\ell_2} = \|\by\|_{\ell_2}= 1$ and $\rank{\bar\bx} = \bar k$, we have $\sum_{i=1}^{4} \norm{ P_i \by }^2_{\ell_2} = 1 $ and $\sum_{i=1}^{2} \norm{ P_i \bx }^2_{\ell_2} = 1 $. Then 
	\begin{align*}
	(\by^\top \bar\bx)^2 & = \left( \sum_{i=1}^{2} \by^\top P_i \bar\bx \right)^2 \\
	&\le \left( \norm{ P_1 \by }_{\ell_2} \norm{ P_1 \bar\bx }_{\ell_2} + \norm{ P_2 \by }_{\ell_2} \norm{ P_2 \bar\bx }_{\ell_2} \right)^2 \\
	&\le \norm{ P_1 \by }^2_{\ell_2} \norm{ P_1 \bar\bx }^2_{\ell_2} + \norm{ P_2 \by }^2_{\ell_2} \norm{ P_2 \bar\bx }^2_{\ell_2} \\ 
	& \quad + 2 \norm{ P_1 \by }_{\ell_2} \norm{ P_1 \bar\bx }_{\ell_2} \norm{ P_2 \by }_{\ell_2} \norm{ P_2 \bar\bx }_{\ell_2}  \\
	&\le \norm{ P_1 \by }^2_{\ell_2} + \norm{ P_2 \by }^2_{\ell_2} \\
	&\le 1 - \norm{ P_3 \by }^2_{\ell_2} \\
	&\le 1 - k_3/k_1 \norm{ P_1 \by }^2_{\ell_2}, 
	\end{align*}
	where the first equation follows from the fact that $P_i^\top P_j = 0$ for $i \ne j$ and $P_3 \bar\bx = P_4 \bar\bx = 0$. The first inequality follows from Cauchy-Schwarz inequality and the third inequality follow from $2 \norm{ P_1 \by }_{\ell_2} \norm{ P_1 \bar\bx }_{\ell_2} \norm{ P_2 \by }_{\ell_2} \norm{ P_2 \bar\bx }_{\ell_2} \le \norm{ P_1 \by }_{\ell_2}^2 \norm{ P_2 \bar\bx }^2_{\ell_2} + \norm{ P_2 \by }^2_{\ell_2} \norm{ P_1 \bar\bx }^2_{\ell_2}$ and $\norm{ P_1 \bar\bx }^2_{\ell_2} +\norm{ P_2 \bar\bx }^2_{\ell_2} = 1$. The last inequality comes from $\norm{P_3\by}^2_{\ell_2} / k_3 \ge \norm{P_1\by}^2_{\ell_2} / k_1$, since $d_1 \ge \cdots \ge d_p$.
	
	This implies that 
	\begin{align} 
	\norm{ P_1 \by }^2_{\ell_2} & \le (k_1/k_3)(1-(\by^\top \bar\bx)^2) \nonumber \\
	& \le  (\bar k/k)(1-(\by^\top \bar\bx)^2), \label{proof:norm2sqofP1y}
	\end{align}
	where the second inequality follows from $\bar k \le k$. Then we obtain that 
	\begin{align} \label{proof:norm2sqofP1x}
	\norm{P_1 \bx}_{\ell_2} \le \min \left[1, (1+(\bar k / k)^{1/2}) \sqrt{1-(\by^\top \bar\bx)^2} \right] 
	\end{align}
	
	Finally from (\ref{proof:norm2sqofP1y}) and (\ref{proof:norm2sqofP1x}), 
	\begin{align*}
	& | \by^\top \bar\bx | - |RankTrunc(\by, k)^\top \bar\bx| \\
	& \le |(\by - RankTrunc(\by, k))^\top \bar\bx| \\
	& = |\by^\top P_{V_k \otimes U_k}^\perp  P_{P_{\bar k} \otimes Q_{\bar k}} \bar\bx| \\
	& = |\by^\top P_1 \bar\bx| \\
	& \le \norm{P_1 \by}_{\ell_2}  \norm{P_1 \bar\bx}_{\ell_2} \\
	& \le (\bar k/k)^{1/2} \min \left[\sqrt{1-(\by^\top \bar\bx)^2}, (1+(\bar k / k)^{1/2}) \left( 1-(\by^\top \bar\bx)^2 \right) \right] 
	\end{align*}
\end{proof}

Next lemma says that each step of the \textsc{smart-pm} algorithm improves eigenvector estimation. 

\begin{lemma} \label{lemma:MatRankTrun_iter_better}
	Assume that $k \ge \bar k$. If $|\langle \bX_{t-1},  \bbX \rangle| > \theta + \delta$ for $\theta \in (0,1)$, then the iterate $\bx_t$ and $\bx_{t-1}$ satisfy the following inequality, 
	\begin{equation*}
	|1 - \langle \bX_t,  \bbX \rangle| \le \mu |1 - \langle \bX_{t-1},  \bbX \rangle|  + \sqrt{10} \delta,
	\end{equation*}
	where $\gamma$ and $\delta$ are defined in (\ref{eqn:eigen_gap_A}) and (\ref{eqn:err_from_sig_noise}), respectively, and $$\mu = \sqrt{1 + 2((\bar k / k)^{1/2} + \bar k / k)} \sqrt{1- 0.5 \theta (1+\theta) (1-\gamma^2)}.$$ 
\end{lemma} 

\begin{proof}
	Let $\bU = \bU_{t-1} \cup \bU_t \cup \bbU$ and $\bV = \bV_{t-1} \cup \bV_t \cup \bbV$. Consider the following vector
	\begin{equation} \label{eqn:algrm_eq}
	\by_t = \bA_{UV} \bx_{t-1} / \norm{\bA_{UV} \bx_{t-1}}_{\ell_2},
	\end{equation}
	where $\bA_{UV} \coloneqq P_{V \otimes U} \bbA P_{V \otimes U}$ denotes the matrix after project the rows and columns of $\bbA$ on the spaces spanned by $\bV \otimes \bU$, respectively. We note that replacing $\bX_t$ with $\mat{\by_t}$ in Algorithm \ref{alg:mlrpower} does not affect the output iteration sequence $\{ \bX_t \}$ because of the low-rankness of $\bX_{t-1}$ and the fact that the truncation operation is invariant to scaling. Therefore for notation simplicity, in the following proof we will simply assume that $\bX_t$ is redefined as $\bX_t = \mat{\by_t}$ according to (\ref{eqn:algrm_eq}). 
	
	Let $\bx_1(\bA_{UV})$ be the eigenvector with the largest (in absolute value) eigenvalue of $\bA_{UV}$. Without loss of generality and for simplicity, we may assume that $\by_t \bx_1(\bA_{UV}) \ge 0$ and $\by_{t-1} \bbx \ge 0$ because otherwise we can simply do appropriate sign changes in the proof. From Lemma \ref{lemma:convg_of_power_method}, we have
	\begin{align*}
	&\by_t^\top \bx_1(\bA_{UV}) \\
	& \ge \bx_{t-1}^\top \bx_1(\bA_{UV}) [1 + (1-\gamma^2)(1-(\bx_{t-1}^\top \bx_1(\bA_{UV}))^2)/2]
	\end{align*}
	
	By assumption that $|\bx_{t-1}^\top \bbx| > \theta + \delta$ and Lemma \ref{lemma:perturbation_rank_truncated}, we have
	\begin{equation*}
	\bx_{t-1}^\top \bx_1(\bA_{UV}) \ge \bx_{t-1}^\top \bbx - \delta \ge \theta
	\end{equation*}
	Thus we have
	\begin{align*}
	&1 - \by_t^\top \bx_1(\bA_{UV}) \\
	&\le \left(1- \bx_{t-1}^\top \bx_1(\bA_{UV}) \right) [1- (1-\gamma^2) \cdot \\
	& \qquad \qquad (1+\bx_{t-1}^\top \bx_1(\bA_{UV}))\bx_{t-1}^\top \bx_1(\bA_{UV})/2] \\
	& \le \left(1- \bx_{t-1}^\top \bx_1(\bA_{UV}) \right) [1- 0.5 \theta (1+\theta) (1-\gamma^2)],
	\end{align*}
	where $\gamma$ is defined in Lemma \ref{lemma:perturbation_rank_truncated}. Using the fact that $\norm{\by_t}_{\ell_2}=\norm{\bx_1(\bA_{UV})}_{\ell_2}=1$, the above result is equivalent to
	\begin{align*}
	&\norm{\by_t - \bx_1(\bA_{UV})}_{\ell_2} \\
	&\le \norm{\bx_{t-1} - \bx_1(\bA_{UV})}_{\ell_2} [1- 0.5 \theta (1+\theta) (1-\gamma^2)].
	\end{align*}
	Using Lemma \ref{lemma:perturbation_rank_truncated}, 
	\begin{align*}
	\norm{\by_t - \bbx}_{\ell_2} & \le \norm{\bx_{t-1} - \bbx}_{\ell_2} [1- 0.5 \theta (1+\theta) (1-\gamma^2)] \\
	& + \delta [2 - 0.5 \theta (1+\theta) (1-\gamma^2)] \\
	& \le \norm{\bx_{t-1} - \bbx}_{\ell_2} [1- 0.5 \theta (1+\theta) (1-\gamma^2)] + 2 \delta
	\end{align*}
	This is equivalent to 
	\begin{align*}
	&\sqrt{1 - |\by_t^\top \bbx|} \\
	&\le \sqrt{1 - |\bx_{t-1}^\top \bbx|} \sqrt{1- 0.5 \theta (1+\theta) (1-\gamma^2)} + \sqrt{2} \delta
	\end{align*}
	
	Applying Lemma \ref{lemma:rank_trunc_err} and using $k > \bar k$, 
	if $(1+(\bar k / k)^{1/2}) \sqrt{1-(\by_t^\top \bar\bx)^2} \le 1$, we have
	\begin{align*}
	&\sqrt{1 - |RankTrunc(\by_t, k)^\top \bbx|} \\
	&\le \sqrt{1 - |\by_t^\top \bbx| + ((\bar k / k)^{1/2} + \bar k / k)(1 - |\by_t^\top \bbx|^2)} \\
	&\le \sqrt{1 - |\by_t^\top \bbx|} \sqrt{1 + 2((\bar k / k)^{1/2} + \bar k / k)}\\
	&\le \mu \sqrt{1 - |\bx_{t-1}^\top \bbx|} + \sqrt{10} \delta,
	\end{align*}
	where
	\[\mu = \sqrt{1 + 2((\bar k / k)^{1/2} + \bar k / k)} \sqrt{1- 0.5 \theta (1+\theta) (1-\gamma^2)}.\] 
	
\end{proof}

\end{appendices}

%% file: lrpca.bbl
\begin{thebibliography}{}

\bibitem[\protect\citeauthoryear{Anderson}{Anderson}{1963}]{anderson1963asymptotic}
Anderson, T.~W. (1963).
\newblock Asymptotic theory for principal component analysis.
\newblock {\em The Annals of Mathematical Statistics\/}.

\bibitem[\protect\citeauthoryear{Berthet and Rigollet}{Berthet and
  Rigollet}{2013}]{berthet2013optimal}
Berthet, Q. and P.~Rigollet (2013).
\newblock Optimal detection of sparse principal components in high dimension.
\newblock {\em The Annals of Statistics\/}~{\em 41\/}(4), 1780--1815.

\bibitem[\protect\citeauthoryear{Birnbaum, Johnstone, Nadler, and
  Paul}{Birnbaum et~al.}{2013}]{birnbaum2013minimax}
Birnbaum, A., I.~M. Johnstone, B.~Nadler, and D.~Paul (2013).
\newblock Minimax bounds for sparse pca with noisy high-dimensional data.
\newblock {\em Annals of statistics\/}~{\em 41\/}(3), 1055.

\bibitem[\protect\citeauthoryear{Brennan, Bresler, and Huleihel}{Brennan
  et~al.}{2018}]{brennan2018reducibility}
Brennan, M., G.~Bresler, and W.~Huleihel (2018).
\newblock Reducibility and computational lower bounds for problems with planted
  sparse structure.
\newblock {\em arXiv preprint arXiv:1806.07508\/}.

\bibitem[\protect\citeauthoryear{Cai, Ma, and Wu}{Cai
  et~al.}{2013}]{cai2013sparse}
Cai, T.~T., Z.~Ma, and Y.~Wu (2013).
\newblock Sparse pca: Optimal rates and adaptive estimation.
\newblock {\em The Annals of Statistics\/}.

\bibitem[\protect\citeauthoryear{Cai, Ren, and Zhou}{Cai
  et~al.}{2013}]{cai2013optimal}
Cai, T.~T., Z.~Ren, and H.~H. Zhou (2013).
\newblock Optimal rates of convergence for estimating toeplitz covariance
  matrices.
\newblock {\em Probability Theory and Related Fields\/}.

\bibitem[\protect\citeauthoryear{Christensen}{Christensen}{2007}]{christensen2007algorithm}
Christensen, L.~P. (2007).
\newblock An em-algorithm for band-toeplitz covariance matrix estimation.
\newblock In {\em Acoustics, Speech and Signal Processing, 2007. ICASSP 2007.
  IEEE International Conference on}, Volume~3, pp.\  III--1021. IEEE.

\bibitem[\protect\citeauthoryear{Davis}{Davis}{2012}]{davis2012circulant}
Davis, P.~J. (2012).
\newblock {\em Circulant matrices}.
\newblock American Mathematical Soc.

\bibitem[\protect\citeauthoryear{Fan, Sun, Zhou, and Zhu}{Fan
  et~al.}{2018}]{fan2018principal}
Fan, J., Q.~Sun, W.-X. Zhou, and Z.~Zhu (2018).
\newblock Principal component analysis for big data.
\newblock {\em arXiv preprint arXiv:1801.01602\/}.

\bibitem[\protect\citeauthoryear{Gray}{Gray}{2006}]{gray2006toeplitz}
Gray, R.~M. (2006).
\newblock Toeplitz and circulant matrices: A review.
\newblock {\em Foundations and Trends{\textregistered} in Communications and
  Information Theory\/}~{\em 2\/}(3), 155--239.

\bibitem[\protect\citeauthoryear{Izenman}{Izenman}{2008}]{izenman2008modern}
Izenman, A.~J. (2008).
\newblock Modern multivariate statistical techniques.
\newblock {\em Regression, classification and manifold learning\/}.

\bibitem[\protect\citeauthoryear{Johnstone and Lu}{Johnstone and
  Lu}{2012}]{johnstone2012consistency}
Johnstone, I.~M. and A.~Y. Lu (2012).
\newblock On consistency and sparsity for principal components analysis in high
  dimensions.
\newblock {\em Journal of the American Statistical Association\/}.

\bibitem[\protect\citeauthoryear{Johnstone and Paul}{Johnstone and
  Paul}{2018}]{johnstone2018pca}
Johnstone, I.~M. and D.~Paul (2018).
\newblock Pca in high dimensions: An orientation.
\newblock {\em Proceedings of the IEEE\/}~{\em 106\/}(8), 1277--1292.

\bibitem[\protect\citeauthoryear{Jolliffe}{Jolliffe}{2011}]{jolliffe2011principal}
Jolliffe, I. (2011).
\newblock Principal component analysis.
\newblock In {\em International encyclopedia of statistical science}, pp.\
  1094--1096. Springer.

\bibitem[\protect\citeauthoryear{Jung and Marron}{Jung and
  Marron}{2009}]{jung2009pca}
Jung, S. and J.~S. Marron (2009).
\newblock Pca consistency in high dimension, low sample size context.
\newblock {\em The Annals of Statistics\/}.

\bibitem[\protect\citeauthoryear{Lei and Vu}{Lei and
  Vu}{2015}]{lei2015sparsistency}
Lei, J. and V.~Q. Vu (2015).
\newblock Sparsistency and agnostic inference in sparse pca.
\newblock {\em The Annals of Statistics\/}.

\bibitem[\protect\citeauthoryear{Ma and Wigderson}{Ma and
  Wigderson}{2015}]{ma2015sum}
Ma, T. and A.~Wigderson (2015).
\newblock Sum-of-squares lower bounds for sparse pca.
\newblock In {\em Advances in Neural Information Processing Systems}, pp.\
  1612--1620.

\bibitem[\protect\citeauthoryear{Ma and Wu}{Ma and
  Wu}{2015}]{ma2015computational}
Ma, Z. and Y.~Wu (2015).
\newblock Computational barriers in minimax submatrix detection.
\newblock {\em The Annals of Statistics\/}~{\em 43\/}(3), 1089--1116.

\bibitem[\protect\citeauthoryear{Muirhead}{Muirhead}{2009}]{muirhead2009aspects}
Muirhead, R.~J. (2009).
\newblock Aspects of multivariate statistical theory.
\newblock {\em John Wiley \& Sons\/}.

\bibitem[\protect\citeauthoryear{Nadler}{Nadler}{2008}]{nadler2008finite}
Nadler, B. (2008).
\newblock Finite sample approximation results for principal component analysis:
  A matrix perturbation approach.
\newblock {\em The Annals of Statistics\/}.

\bibitem[\protect\citeauthoryear{Roberts and Ephraim}{Roberts and
  Ephraim}{2000}]{roberts2000hidden}
Roberts, W.~J. and Y.~Ephraim (2000).
\newblock Hidden markov modeling of speech using toeplitz covariance matrices.
\newblock {\em Speech Communication\/}~{\em 31\/}(1), 1--14.

\bibitem[\protect\citeauthoryear{Snyder, O'Sullivan, and Miller}{Snyder
  et~al.}{1989}]{snyder1989use}
Snyder, D.~L., J.~A. O'Sullivan, and M.~I. Miller (1989).
\newblock The use of maximum likelihood estimation for forming images of
  diffuse radar targets from delay-doppler data.
\newblock {\em IEEE Transactions on Information Theory\/}~{\em 35\/}(3),
  536--548.

\bibitem[\protect\citeauthoryear{Vu and Lei}{Vu and Lei}{2012}]{vu2012minimax}
Vu, V.~Q. and J.~Lei (2012).
\newblock Minimax rates of estimation for sparse pca in high dimensions.
\newblock In {\em International Conference on Artificial Intelligence and
  Statistics}.

\bibitem[\protect\citeauthoryear{Vu and Lei}{Vu and Lei}{2013}]{vu2013minimax}
Vu, V.~Q. and J.~Lei (2013).
\newblock Minimax sparse principal subspace estimation in high dimensions.
\newblock {\em The Annals of Statistics\/}.

\bibitem[\protect\citeauthoryear{Wang, Berthet, and Samworth}{Wang
  et~al.}{2016}]{wang2016statistical}
Wang, T., Q.~Berthet, and R.~J. Samworth (2016).
\newblock Statistical and computational trade-offs in estimation of sparse
  principal components.
\newblock {\em The Annals of Statistics\/}~{\em 44\/}(5), 1896--1930.

\bibitem[\protect\citeauthoryear{Werner, Jansson, and Stoica}{Werner
  et~al.}{2008}]{werner2008estimation}
Werner, K., M.~Jansson, and P.~Stoica (2008).
\newblock On estimation of covariance matrices with kronecker product
  structure.
\newblock {\em IEEE Transactions on Signal Processing\/}~{\em 56\/}(2),
  478--491.

\bibitem[\protect\citeauthoryear{Yuan and Zhang}{Yuan and
  Zhang}{2013}]{yuan2013truncated}
Yuan, X.-T. and T.~Zhang (2013).
\newblock Truncated power method for sparse eigenvalue problems.
\newblock {\em Journal of Machine Learning Research\/}~{\em 14\/}(Apr),
  899--925.

\bibitem[\protect\citeauthoryear{Zou, Hastie, and Tibshirani}{Zou
  et~al.}{2006}]{zou2006sparse}
Zou, H., T.~Hastie, and R.~Tibshirani (2006).
\newblock Sparse principal component analysis.
\newblock {\em Journal of computational and graphical statistics\/}~{\em
  15\/}(2), 265--286.

\end{thebibliography}
